\documentclass[journal]{IEEEtran}
\usepackage{amsthm}
\usepackage{amsfonts}
\usepackage{amssymb}
\usepackage{hyperref}
\usepackage{graphicx}
\usepackage{enumerate}
\usepackage{amsmath}
\usepackage{color}
\usepackage{algpseudocode}
\usepackage{algorithm}
\usepackage{array}

\newcolumntype{L}[1]{>{\raggedright\let\newline\\\arraybackslash\hspace{0pt}}m{#1}}
\newcolumntype{C}[1]{>{\centering\let\newline\\\arraybackslash\hspace{0pt}}m{#1}}
\newcolumntype{R}[1]{>{\raggedleft\let\newline\\\arraybackslash\hspace{0pt}}m{#1}}

\newcommand{\bp}{\begin{proof} \small }
\newcommand{\ep}{\end{proof} \normalsize}
\newcommand{\epx}{\end{proof} \small}
\newcommand{\bpa}{\begin{proofappx} \footnotesize }
\newcommand{\epa}{\end{proofappx} \small }
\newtheorem{theorem}{Theorem}

\newtheorem{lemma}{Lemma}

\newtheorem*{theorem*}{Theorem}
\newtheorem*{proposition*}{Proposition}
\newtheorem*{corollary*}{Corollary}
\newtheorem*{lemma*}{Lemma}
\newtheorem*{assumption*}{Assumption}
\newtheorem*{definition*}{Definition}
\newtheorem*{claim*}{Claim}

\newcommand{\be}{\begin{equation}}
\newcommand{\ee}{\end{equation}}
\newcommand{\bs}{\begin{subequations}}
\newcommand{\es}{\end{subequations}}
\newcommand{\bq}{\begin{eqnarray}}
\newcommand{\eq}{\end{eqnarray}}
\newcommand{\bqn}{\begin{eqnarray*}}
\newcommand{\eqn}{\end{eqnarray*}}

\newcommand{\ba}{\left[ \begin{array}}
\newcommand{\ea}{\\ \end{array} \right]}
\newcommand{\ben}{\begin{enumerate}}
\newcommand{\een}{\end{enumerate}}

\def\a{{\boldsymbol{a}}}

\def\x{{\boldsymbol{x}}}

\def\real{{\mathchoice%
{\hbox{\rm\setbox1=\hbox{I}\copy1\kern-.45\wd1 R}}
{\hbox{\rm\setbox1=\hbox{I}\copy1\kern-.45\wd1 R}}
{\hbox{\scriptsize\rm\setbox1=\hbox{I}\copy1\kern-.45\wd1 R}}
{\hbox{\scriptsize\rm\setbox1=\hbox{I}\copy1\kern-.45\wd1 R}}}}

\def\Zint{{\mathchoice{\setbox1=\hbox{\sf Z}\copy1\kern-.75\wd1\box1}
{\setbox1=\hbox{\sf Z}\copy1\kern-.75\wd1\box1}
{\setbox1=\hbox{\scriptsize\sf Z}\copy1\kern-.75\wd1\box1}
{\setbox1=\hbox{\scriptsize\sf Z}\copy1\kern-.75\wd1\box1}}}
\newcommand{\complex}{ \hbox{\rm C\kern-0.45em\rule[.07em]{.02em}{.58em}%
\kern 0.43em}}

\begin{document}
%
\title{Forecasting Popularity of Videos \\using Social Media}
%
%
%

\author{Jie~Xu,~\IEEEmembership{Student Member,~IEEE,}
        Mihaela~van~der~Schaar,~\IEEEmembership{Fellow,~IEEE,}\\
        Jiangchuan~Liu,~\IEEEmembership{Senior Member,~IEEE,}
        and~Haitao~Li,~\IEEEmembership{Student Member,~IEEE}
        \thanks{J. Xu and M. van der Schaar are with the Department of Electrical Engineering, University of California, Los Angeles, USA. (email: jiexu@ucla.edu; miheala@ee.ucla.edu)}
        \thanks{J. Liu and H. Li are with the School of Computing Science, Simon Fraser University, Burnaby, Canada. (email: jcliu@cs.sfu.ca; haitaol@sfu.ca.}
}

\maketitle

\begin{abstract}
This paper presents a systematic online prediction method (Social-Forecast) that is capable to accurately forecast the popularity of videos promoted by social media. Social-Forecast explicitly considers the dynamically changing and evolving propagation patterns of videos in social media when making popularity forecasts, thereby being situation and context aware.  Social-Forecast aims to maximize the forecast reward, which is defined as a tradeoff between the popularity prediction accuracy and the timeliness with which a prediction is issued. The forecasting is performed online and requires no training phase or a priori knowledge. We analytically bound the prediction performance loss of Social-Forecast as compared to that obtained by an omniscient oracle and prove that the bound is sublinear in the number of video arrivals, thereby guaranteeing its short-term performance as well as its asymptotic convergence to the optimal performance. In addition, we conduct extensive experiments using real-world data traces collected from the videos shared in RenRen, one of the largest online social networks in China. These experiments show that our proposed method outperforms existing view-based approaches for popularity prediction (which are not context-aware) by more than 30\% in terms of prediction rewards.
\end{abstract}


%
\IEEEpeerreviewmaketitle

\section{Introduction}
Networked services in the Web 2.0 era focus increasingly on the user participation in producing and interacting with rich media. The role of the Internet itself has evolved from the original use as a communication infrastructure, where users passively receive and consume media content to a social ecosystem, where users equipped with mobile devices constantly generate  media  data through a variety of sensors (cameras, GPS, accelerometers, etc.) and applications and, subsequently, share this acquired data through social media. Hence, social media is recently being used to provide situational awareness and inform predictions and decisions in a variety of application domains, ranging from live or on-demand event broadcasting, to security and surveillance~\cite{Trottier}, to health communication~\cite{Chou}, to disaster management~\cite{Sakaki},  to economic forecasting~\cite{Choi}. In all these applications, forecasting the popularity of the content shared in a social network is vital due to a variety of reasons. For network and cloud service providers, accurate forecasting facilitates prompt and adequate reservation of computation, storage, and bandwidth resources~\cite{Liu}, thereby ensuring smooth and robust content delivery at low costs. For advertisers, accurate and timely popularity prediction provides a good revenue indicator, thereby enabling targeted ads to be composed for specific videos and viewer demographics. For content producers and contributors, attracting a high number of views is paramount for attracting potential revenue through micro-payment mechanisms.

While popularity prediction is a long-lasting research topic~\cite{Szabo}~\cite{Cha}~\cite{Wu}~\cite{Pinto}, understanding how social networks affect the popularity of the media content and using this understanding to make better forecasts poses significant new challenges. Conventional prediction tools have mostly relied on the history of the past view counts, which worked well when the popularity solely depended on the inherent attractiveness of the content and the recipients were generally passive. In contrast, social media users are proactive in terms of the content they watch and are heavily influenced by their social media interactions; for instance, the recipient of a certain media content may further forward it or not, depending on not only its attractiveness, but also the situational and contextual conditions in which this content was generated and propagated through social media~\cite{Li}. For example, the latest measurement on Twitter's Vine, a highly popular short mobile video sharing service, has suggested that the popularity of a short video indeed depends less on the content itself, but more on the contributor's position in the social network~\cite{Zhang}. Hence, being situation-aware, e.g. considering the content initiator's information and the friendship network of the sharers, can clearly improve the accuracy of the popularity forecasts. However, critical new questions need to be answered: which situational information extracted from social media should be used, how to deal with dynamically changing and evolving situational information, and how to use this information efficiently to improve the forecasts?

As social media becomes increasingly more ubiquitous and influential, the video propagation patterns and users' sharing behavior dynamically change and evolve as well. Offline prediction tools~\cite{Szabo}~\cite{Galuba}~\cite{Hong}~\cite{Lerman} depend on specific training datasets, which may be biased or outdated, and hence may not accurately capture the real-world propagation patterns promoted by social media. Moreover, popularity forecasting is a \textit{multi-stage} rather than a single-stage task since each video may be propagated through a cascaded social network for a relatively long time and thus, the forecast can be made at any time while the video is being propagated. A fast prediction has important economic and technological benefits; however, too early a prediction may lead to a low accuracy that is less useful or even damaging (e.g. investment in videos that will not actually become popular). The timeliness of the prediction has yet to be considered in existing works~\cite{Cha}-\cite{Lerman}~\cite{Wu}~\cite{Pinto} which solely focus on maximizing the accuracy. Hence, we strongly believe that developing a systematic methodology for accurate and timely popularity forecasting is essential.

In this paper, we propose for the first time a systematic methodology and associated online algorithm for forecastingl popularity of videos promoted by social media. Our Social-Forecast algorithm is able to make predictions about the popularity of videos while jointly considering the accuracy and the timeliness of the prediction. We explicitly consider the unique situational conditions that affect the video propagated in social media, and demonstrate how this \textit{context information} can be incorporated to improve the accuracy of the forecasts. The unique features of Social-Forecast as well as our key contributions are summarized below:

\begin{itemize}
\item We rigorously formulate the online popularity prediction as a multi-stage sequential decision and online learning problem. Our solution, the Social-Forecast algorithm, makes multi-level popularity prediction in an online fashion, requiring no {\em a priori} training phase or dataset. It exploits the dynamically changing and evolving video propagation patterns through social media to maximize the prediction reward. The algorithm is easily tunable to enable tradeoffs between the accuracy and timeliness of the forecasts as required by various applications, entities and/or deployment scenarios.
\item We analytically quantify the regret of Social-Forecast, that is, the performance gap between its expected reward and that of the best prediction policy which can be only obtained by an omniscient oracle having complete knowledge of the video popularity trends. We prove that the regret is sublinear in the number of video arrivals, which implies that the expected prediction reward asymptotically converges to the optimal expected reward. The upper bound on regret also gives a lower bound on the convergence rate to the optimal average reward.
\item We validate Social-Forecast's performance through extensive experiments with real-world data traces from RenRen (the largest Facebook-like online social network in China). The results show that significant improvement can be achieved by exploiting the situational and contextual meta-data associated with the video and its propagation through the social media. Specifically, the Social-Forecast algorithm outperforms existing view-based approaches by more than 30\% in terms of prediction rewards.
\end{itemize}
The rest of the paper is organized as follows. Section II discusses related works. In Section III, we describe the system model and rigorously formulate the online popularity prediction problem. Section IV presents a systematic methodology for determining the optimal prediction policy with complete prior knowledge of the video propagation pattern. In Section V, we propose the online learning algorithm for the optimal prediction policy and prove that it achieves sublinear regret bounds. Section VI discusses the experimental results and our findings. Section VII concludes this paper.

\section{Related Works}
In this section, we review the representative related works from both the application and the theoretical foundation perspectives.

\subsection{Popularity Prediction for Online Content}
Popularity prediction of online content has been extensively studied in the literature. Early works have focused on predicting the future popularity of content (e.g. video) on conventional websites such as YouTube. Various solutions are proposed based on time series models like ARMA (Autoregressive moving average)~\cite{Niu}~\cite{Glzrsun}~\cite{Amondeo}, regression models~\cite{Wang}~\cite{Lee}~\cite{Rowe} and classification models~\cite{Wang}~\cite{Shamma}~\cite{Siersdorfer}. These methods are generally view-based, meaning that the prediction of the future views is solely based on the early views, while disregarding the situational context during propagation. For instance, it was found that a high correlation exists between the number of video views on early days and later days on YouTube~\cite{Cha}. By using the history of views within the past 10 days, the popularity of videos can be predicted up to 30 days ahead~\cite{Szabo}. While these predictions methods provide satisfactory performance for YouTube-like accesses, their performance is largely unacceptable~\cite{Li} when applied to predicting popularity in the social media context. This is because in this case the popularity of videos evolves in a significantly different manner which is highly influenced by the situational and contextual characteristics of the social networks in which the video has propagated~\cite{Li-IWQoS}.

Recently, there have been numerous studies aiming to accurately predicting the popularity of content promoted by social media~\cite{Chou}~\cite{Sakaki}~\cite{Asur}-\cite{Kooti}. For instance, a propagation model is proposed in~\cite{Galuba} to predict which users are likely to mention which URLs on Twitter. In~\cite{Hong}, the retweets prediction on Twitter is modeled as a classification problem, and a variety of context-aware features are investigated. For predicting the popularity of news in Digg, such aspects as website design have been incorporated~\cite{Lerman}, and for predicting the popularity of short messages, the structural characteristics of social media have been used~\cite{Bao}. For video sharing in social media, our earlier work~\cite{Li} has identified a series of context-aware factors which influence the propagation patterns.

Our work in this paper is motivated by these studies, but it is first systematic solution for forecasting the video popularity based on the situational and contextual characteristics of social media. First, existing works are mostly measurement-based and their solutions generally work offline, requiring existing training data sets. Instead, Social-Forecast operates entirely online and  does not require any a priori gathered training data set. Second, Social-Forecast is situation-aware and hence it can inherently adapt on-the-fly to the underlying social network structure and user sharing behavior. Last but not least, unlike the early empirical studies which employ only simulations to validate the performance of their predictions, we can rigorously prove performance bounds for Social-Forecast.

Importantly, our Social-Forecast can be easily extended to predict other trends in social media (such as predicting who are the key influencers in social networks, which tweets and news items may become viral, which content may become popular or relevant etc.) by exploiting contextual and situational awareness. For instance, besides popularity, social media has been playing an increasingly important role in predicting present or near future events. Early studies show that the volume and the frequency of Twitter posts can be used to forecast box-office revenues for movies~\cite{Asur} and detect earthquakes~\cite{Sakaki}. Sentiment detection is investigated in~\cite{Barbosa} by exploring characteristics of how tweets are written and meta-information of the words that compose these messages. In~\cite{Chou}, Google Trends uses search engine data to forecast near-term values of economic indicators, such as automobile sales, unemployment claims, travel destination planning, and consumer confidence. Social-Forecast can be easily adapted for deployment in these applications as well.

Table \ref{comparisontable} provides a comprehensive comparison between existing works on popularity prediction and Social-Forecast, highlighting their differences.

\subsection{Quickest Detection and Contextual Bandits Learning}
In our problem formulation, for each video, the algorithm can choose
to make a prediction decision using the currently observed context
information or wait to make this prediction until the next period,
when more context information arrives. This introduces a tradeoff
between accuracy and delay which relates to the literature on quickest
detection~\cite{Poor}~\cite{Krishnamruthy}~\cite{Lai} which
is concerned with the problem of detecting the change in the underlying
state (which has already occurred in the past). For example,
authors in~\cite{Lai} study how to detect the presence of primary
users by taking channel sensing samples in cognitive radio systems.
In the considered problem, there is no underlying state; in fact, the state is continuously and dynamically changing, and the problem becomes forecasting how it will evolve and which event will occur in the future. Moreover, many quickest detection solutions assume prior knowledge of the hypotheses
~\cite{Lai} while this knowledge is unknown a priori in our problem
and needs to be discovered over time to make accurate forecasts.

Our forecasting algorithm is based on the contextual bandits framework
~\cite{Tekin}-\cite{Chu} but with
significant innovations aimed at tackling the unique features of the
online prediction problem. First, most of the prior work~\cite{Slivkins}-\cite{Chu}
on contextual bandits is focused on an agent making a single-stage
decision based on the provided context information for each incoming
instance. In this paper, for each incoming video instance, the agent
needs to make a sequence of decisions at multiple stages. The context
information is stage-dependent and is revealed only when that stage
takes place. Importantly, the reward obtained by selecting an action
at one stage depends on the actions chosen at other stages and thus,
rewards and actions at different stages are coupled. Second, in existing
works~\cite{Tekin}-\cite{Chu}, the estimated rewards of an action can be updated only after the action is selected. In our problem, because the prediction action does not
affect the underlying popularity evolution, rewards can be computed
and updated even for actions that are not selected. In particular, we update the reward of an action as if it was selected. Therefore, exploration
becomes virtual in the sense that explicit explorations are not needed and hence, in each period, actions with the best estimated rewards can always be selected, thereby improving the learning performance.

\begin{table}
  \centering
  \includegraphics[scale=0.7]{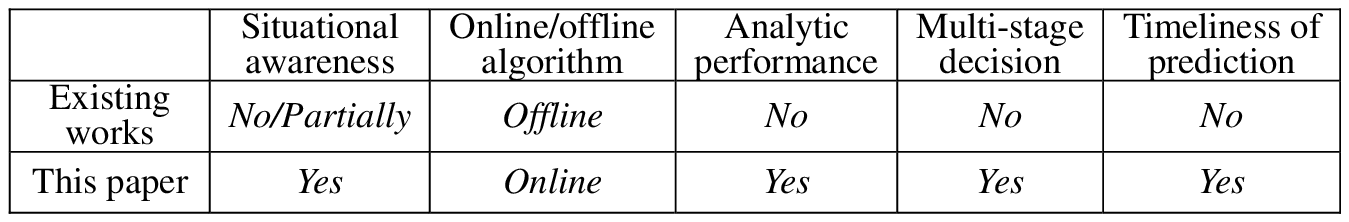}\\
  \caption{Comparison with existing works on popularity prediction for online content.}\label{comparisontable}
\end{table}

\section{System Model}
\begin{figure*}[t]
  \centering
  \includegraphics[scale=0.8]{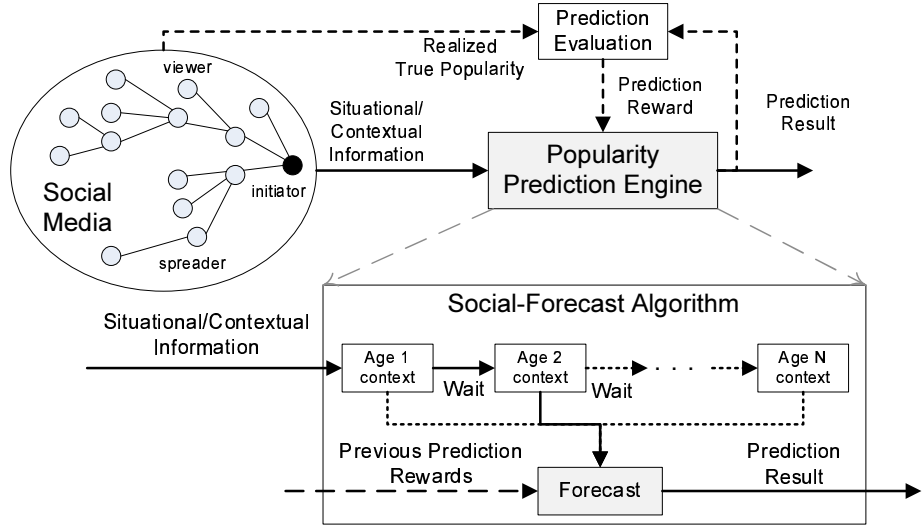}\\
  \caption{System diagram.}\label{system}
\end{figure*}

\subsection{Sharing Propagation and Popularity Evolution}
We consider a generic Web 2.0 information sharing system in which videos are shared
by users through social media (see Figure \ref{system} for
a system diagram). We assign each video with an index $k\in\{1,2,...,K\}$
according to the absolute time $t_{init}^{k}$ when it is initiated\footnote{It is easy to assign unique identifiers if multiplel videos which are generated/initiated at the same time.}. Once a video is initiated, it will be propagated through the social media for some time duration. We assume a discrete time model where a period can be minutes, hours, days, or any suitable time duration. A video is said to have an age of $n\in\{1,2,...\}$ periods if it has been propagated through the social media for $n$ periods. In each period, the video is further shared and viewed by users depending on the sharing and viewing status of the previous period. The propagation characteristics of video $k$ up to age $n$ are captured by a $d_{n}$-dimensional vector $\x_{n}^{k}\in\mathcal{X}_{n}$ which includes information such as the total number of views and other situational and contextual information such as the characteristics
of the social network over which the video was propagated. The specific
characteristics that we use in this paper will be discussed in Section
VI. In this section, we keep $\x_{n}^{k}$ in an abstract form and
call it succinctly the \textit{context (and situational) information}
at age $n$.

Several points regarding the context information are noteworthy. First, the context space $\mathcal{X}_{n}$ can be different at different ages $n$. In particular, $\x^k_n$ can include all history information of video $k$'s propagation characteristics up to age $n$ and hence $\x^k_{n}$ includes all information of $\x^k_{m},\forall m < n$ (See Figure \ref{contexthistory}). Thus the type of contextual/situational information is also age-dependent. Second, $\x_{n}^{k}$ can be taken from a large space, e.g. a finite space with a large number of values or even an infinite space. For example, some dimensions of $\x^k_n$ (e.g. the Sharing Rate used in Section VI) take values from a continuous value space and $\x^k_n$ may include all the past propagation characteristics (e.g. $\x^k_{m} \in \x^k_n, \forall m < n$). Third, at age $n$, $\x_{m}^{k},\forall m>n$ are not yet revealed since they represent
future situational and contextual information which is yet to be realized. Hence, given the context information $\x^k_n$ at age $n$, the future context information $\x_{m}^{k},\forall m>n$ are random variables.

\begin{figure}
  \centering
  \includegraphics[scale=0.9]{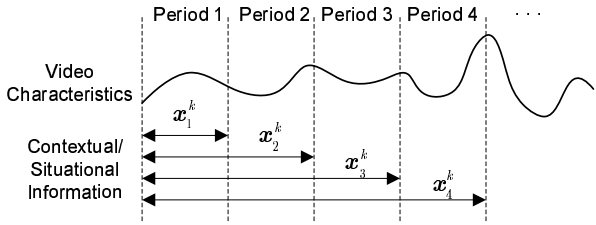}\\
  \caption{An illustration of context information taking the history characteristics.}\label{contexthistory}
\end{figure}

We are interested in predicting the future popularity status of the video by the end of a pre-determined age $N$, and we aim to make the prediction as soon as possible. The choice of $N$ depends on the specific requirements of the content provider, the advertiser and the web hosts. In this paper, we will treat $N$ as given\footnote{This assumption is generally valid given that the video sharing events have daily and weekly patterns, and the active lifespans of most shared videos through social media are quite limited~\cite{Li-IWQoS}}. Thus, the context information for video $k$ during its lifetime of $N$ periods is collected in $\x^{k}=(\x_{1}^{k},\x_{2}^{k},...,\x_{N}^{k})$. For expositional simplicity, we also define $\x_{n^{+}}=(\x_{n+1},...,\x_{N})$, $\x_{n^{-}}=(\x_{1},...,\x_{n-1})$ and $\x_{-n}=(\x_{n^{-}},\x_{n^{+}})$.

Let $\mathcal{S}$ be the popularity status space, which is assumed to be finite. For instance, $\mathcal{S}$ can be either a binary space \{Popular, Unpopular\} or a more refined space containing multiple levels of popularity such as \{Low Popularity, Medium Popularity, High Popularity\} or any such refinement. We let $s^k$ denote the popularity status of video $k$ by the end of age $N$. Since $s^k$ is realized only at the end of $N$ periods, it is a random variable at all previous ages. However, the conditional distribution of $s^k$ will vary at different ages since they are conditioned on different context information. In many scenarios, the conditional distribution at a higher age $n$ is more informative for the future popularity status since more contextual information has arrived. Nevertheless, our model does not require this assumption to hold.

\subsection{Prediction Reward}
For each video $k$, at each age $n=1,...,N$, we can make a prediction decision $a^k_n \in \mathcal{S}\cup \{\textrm{Wait}\}$. If $a^k_n \in \mathcal{S}$, we predict $a^k_n$ as the popularity status by age $N$. If $a^k_n = \textrm{Wait}$, we choose to wait for the next period context information to decide (i.e. predict a popularity status or wait again). For each video $k$, at the end of age $N$, given the decision action vector $\a^k$, we define the {\it age-dependent reward} $r^k_n$ at age $n$ as follows,
\begin{equation}\label{reward}
r^k_n = \left\{
\begin{array}{ll}
U(a^k_n, s^k, n), &\textrm{if}~a^k_n \in \mathcal{S}\\
r^k_{n+1}, &\textrm{if}~a^k_n = \textrm{Wait}
\end{array}\right.
\end{equation}
where $U(a^k_n, s^k, n)$ is a reward function depending on the accuracy of the prediction (determined by $a^k_n$ and the realized true popularity status $s^k$) and the timeliness of the prediction (determined by the age $n$ when the prediction is made).

The specific form of $U(a^k_n, s^k, n)$ depends on how the reward is derived based on the popularity prediction based on various economical and technological factors. For instance, the reward can the ad revenue derived from placing proper ads for potential popular videos or the cost spent for adequately planning computation, storage, and bandwidth resources to ensure the robust operation of the video streaming services. Even though our framework allows any general form of the reward function, in our experiments (Section VI), we will use a reward function that takes the form of $U(a^k_n, s^k, n) = \theta(a^k_n, s^k) + \lambda \psi(n)$ where $\theta(a^k_n, s^k)$ measures the prediction accuracy, $\psi(n)$  accounts for the prediction timeliness and $\lambda > 0$ is a trade-off parameter that controls the relative importance of accuracy and timeliness.

Let $n^*$ be the first age at which the action is not ``Wait'' (i.e. the first time a forecast is issued). The {\it overall prediction reward} is defined as the $r^k = r^k_{n^*}$. According to equation \eqref{reward}, when the action is ``Wait'' at age $n$, the reward is the same as that at age $n+1$. Thus $r^k_{1} = r^k_{2} = ... = r^k_{n^*}$. This suggests that the overall prediction reward is the same as the age-dependent reward at age 1, i.e. $r^k = r^k_1$. For age $n > n^*$, the action $a^n_k$ and the age-dependent reward $r^k_n$ do not affect the realized overall prediction result since a prediction has already been made. However, we still select actions and compute the age-dependent reward since it helps learning the best action and the best reward for this age $n$ which in turn will help decide whether or not we should wait at an early age. Figure \ref{msd} provides an illustration on how the actions at different ages determine the overall prediction reward.

\begin{figure}
  \centering
  \includegraphics[scale=0.9]{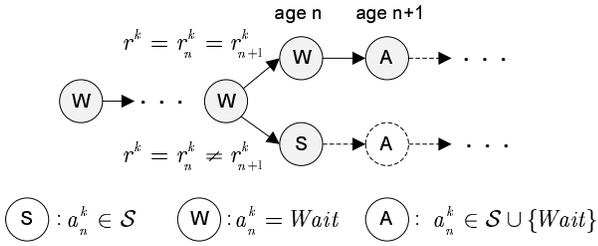}\\
  \caption{An illustration for the multi-stage decision making. The first $n-1$ action is ``Wait''. If the age-$n$ action is ``Wait'', then $r^k_n = r^k_{n+1}$ which depends on later actions. If the age-$n$ action is not ``Wait'', then $r^k_{n} \neq r^k_{n+1}$ and $r^k$ does not depend on later actions. However, we can still learn the reward of action at age $n+1$ as if all actions before $n+1$ were ``Wait''. }\label{msd}
\end{figure}

{\it Remark}: The prediction action itself does not generate rewards. It is the action (e.g. online ad investment) taken using the prediction results that is rewarding. In many scenarios, this action can only be taken once and cannot be altered afterwards. This motivates the above overall reward function formulation in which the overall prediction reward is determined by the first non-``Wait'' action. Nevertheless, our framework can also be easily extended to account for more general overall reward functions which may depend on all non-``Wait'' actions. For instance, the action may be revised when a more accurate later prediction is made. In this case, the reward function $U(a^k_n, s^k, n)$ in \eqref{reward} will depend on not only the current prediction action $a^k_n \in \mathcal{S}$ but also all non-``Wait'' actions after age $n$. We will use the reward function in \eqref{reward} because of its simplicity for the exposition but our analysis also holds for general reward functions.

\subsection{Prediction Policy}
In this paper, we focus on prediction policies that depend on the
current contextual information. Let $\pi_{n}:\mathcal{X}_{n}\to\mathcal{S}\cup \{\textrm{Wait}\}$
denote the prediction policy for a video link of age $n$ and $\pi=(\pi_{1},...,\pi_{N})$
be the complete prediction policy. Hence, a prediction policy $\pi$ prescribes actions for all possible context information at all ages. For expositional simplicity, we also define $\pi_{n^{+}}=(\pi_{n+1},...,\pi_{N})$ as the policy vector
for ages greater than $n$, $\pi_{n^{-}}=(\pi_{1},...,\pi_{n-1})$
as the policy vector for ages smaller than $n$ and $\pi_{-n}=(\pi_{n^{-}},\pi_{n^{+}})$.
For a video with context information $\x^{k}$, the prediction policy
$\pi$ determines the prediction action at each age and hence the overall prediction reward, denoted by $r(\x|\pi)$, as well as the age-dependent rewards $r_n(\x|\pi), \forall n = 1,...,N$. Let $f(\x)$ be the probability
distribution function of the video context information, which also gives information of the popularity evaluation patterns. The expected
prediction reward of a policy $\pi$ is therefore,
\begin{align}
V(\pi) = \int_{\x\in\mathcal{X}} r(\x|\pi) f(\x) d\x
\end{align}
Note that the age-$n$ policy $\pi_n$ will only use the context information $\x_n$ rather than $\x$ to make predictions since $\x_{n^+}$ has not been realized at age $n$.

Our objective is to determine the optimal policy $\pi^{opt}$ that
maximizes the expected prediction reward, i.e. $\pi^{opt}=\arg\max\limits _{\pi}V(\pi)$.
In the following sections, we will propose a systematic methodology
and associated algorithms that find the optimal policy for the case
when $f(\x)$ is known or unknown, which are referred to as the complete
and incomplete information scenarios, respectively.

\section{Why Online Learning is Important?}
In this section, we consider the optimal policy design problem with the complete information of the context distribution $f(\x)$ and compute the optimal policy $\pi^{opt}$. In the next section in which $f(\x)$ is unknown, we will learn this optimal policy $\pi^{opt}$ online and hence, the solution that we derive in this section will serve as the benchmark. Even when having the complete information, determining the optimal prediction policy faces great challenges: first, the prediction reward depends on \textit{all} decision actions at \textit{all} ages; and second, when making the decision at age $n$, the actions for ages larger than $n$ are not known since the corresponding context information has not been realized yet.

Given policies $\pi_{-n}$, we define the expected reward when taking action $a_n$ for $\x_n$ as follows,
\begin{align}
\mu_n(\x'_n|\pi_{-n},a_n) = \int_{\x} I_{\x_n = \x'_n} r_n(\x|\pi_{-n}, a_n) f(\x) d\x
\end{align}
where $I_{\x_n = \x'_n}$ is an indicator function which takes value 1 when the age-$n$ context information is $\x'_n$ and value 0 otherwise. The optimal $\pi^*(\pi_{-n})$ given $\pi_{-n}$ thus can be determined by
\begin{equation}\label{best_response}
\pi_n^*(\x_n|\pi_{-n}) = \arg\max\limits_{a}\mu(\x_n|\pi_{-n},a), \forall \x_n
\end{equation}
and in which we break ties deterministically. Equation \eqref{best_response} defines a best response function from a policy to a new policy $F: \Pi \to \Pi$ where $\Pi$ is the space of all policies. In order to compute the optimal policy $\pi^{opt}$, we iteratively use the best response function in \eqref{best_response} using the output policy computed in the previous iteration as the input for the new iteration. Note that a computation iteration is different from a time period. ``Period'' is used to describe the time unit of the discrete time model of the video propagation. A period can be a miniute, an hour or any suitable time duration. In each period, the sharing and viewing statistics of a {\it specific} video may change. ``Iteration'' is used for the (offline) computation method for the optimal policy (which prescribes actions for {\it all} possible context information in {\it all} periods). Given the complete statistical information (i.e. the video propagation characteristics distribution $f(\x)$) of videos, a new policy is computed using best response update in each iteration.

We prove the convergence and optimality of this best response update as follows.
\begin{lemma}
$\pi_n^*(\x_n|\pi_{-n})$ is independent of $\pi_m, \forall m < n$, i.e. $\pi_n^*(\x_n|\pi_{-n}) = \pi_n^*(\x_n|\pi_{n^+})$.
\end{lemma}
\begin{proof}
By the definition of age-dependent reward, the prediction actions before age $n$ does not affect the age-$n$ reward. Hence, the optimal policy depends only on the actions after age $n$.
\end{proof}

Lemma 1 shows that the optimal policy $\pi_n$ at age $n$ is fully determined by the policies for ages larger than $n$ but does not depend on the policies for ages less than $n$. Using this result, we can show the best response algorithm converges to the optimal policy within a finite number of computation iterations.

\begin{theorem}
Starting with any initial policy $\pi^0$, the best response update converges to a unique point $\pi^*$ in $N$ computation iterations. Moreover, $\pi^* = \pi^{opt}$.
\end{theorem}
\begin{proof}
Given the context distribution $f(\x)$ which also implies the popularity evolution, the optimal age-$N$ policy can be determined in the first iteration. Since we break ties deterministicaly when rewards are the same, the policy is unique. Given this, in the second iteration, the optimal age-$(N-1)$ policy can be determined according to \eqref{best_response} and is also unique. By induction, the best response update determines the unique optimal age-$n$ policy after $N+1-n$ iterations. Therefore, the complete policy is found in $N$ iterations and this policy maximizes the overall prediction reward.
\end{proof}

Theorem 1 proves that we can compute the optimal prediction policy using a simple iterative algorithm as long as we have complete knowledge of the popularity evolution distribution. In practice, this information is unknown and extremely difficult to obtain, if not possible. One way to estimate this information is based on a training set. Since the context space is usually very large (which usually involves infinite number of values), a very large volume of training set is required to obtain a reasonably good estimation. Moreover, existing training sets may be biased and outdated as social media evolves. Hence, prediction policies developed using existing training sets may be highly inefficient~\cite{Sollich}. In the following section, we develop learning algorithms to learn the optimal policy in an online fashion, requiring no initial knowledge of the popularity evolution patterns.

\section{Learning the Optimal Forecasting Policy with Incomplete Information}
In this section, we develop a learning algorithm to determine the optimal prediction policy without any prior knowledge of the underlying context distribution $f(\x)$. In the considered scenario, videos arrive to the system in sequence\footnote{To simplify our analysis, we will assume that one video arrives at one time. Nevertheless, our framework can be easily extended to scenarios where multiple videos arrive at the same time.} and we will make popularity prediction based on past experiences by exploiting the similarity information of videos.

Since we have shown in the last section that we can determine the complete policy $\pi$ using a simple iterative algorithm, we now focus mainly on learning $\pi_n$ for one age by fixing the policies $\pi_{-n}$ for other ages. Importantly, we will provide not only asymptotic convergence results but also prediction performance bounds during the learning process.

\subsection{Learning Regret}
In this subsection, we define the performance metric of our learning algorithm. Let $\sigma_n$ be a learning algorithm of $\pi_n$ which takes action $\sigma^k_n(\x^k_n)$ at instance $k$. We will use learning regret to evaluate the performance of a learning algorithm. Since we focus on $\pi_n$, we will use simplified notations in this section by neglecting $\pi_{-n}$. However, keep in mind that the age-$n$ prediction reward depends on actions at all later ages $a_{n^+}$ besides $a_n$ when $a_n = \textrm{Wait}$. Let $\mu_n(\x_n|a_n)$ denote the expected reward when age-$n$ context information is $\x_n$ and the algorithm takes the action $a_n \in \mathcal{S}\cup \{\textrm{Wait}\}$.  We make a widely adopted assumption~\cite{Slivkins}~\cite{Dudik}~\cite{Langford} that the expected reward of an action is similar for similar contextual and situational information; we formalize this in terms of (uniform) Lipschitz condition.

\begin{assumption*}(Lipschitz)
For each $a_n \in \mathcal{S}\cup \{\textrm{Wait}\}$, there exists $L > 0, \alpha > 0$ such that for all $\x_n, \x'_n \in \mathcal{X}_n$, we have $|\mu(\x_n|a_n) - \mu(\x'_n|a_n)|\leq L\|\x_n, \x'_n\|^\alpha$.
\end{assumption*}

The optimal action given a context $\x_n$ is therefore, $a^*(\x_n) = \arg\max_{a_n} \mu_n(\x_n|a_n)$ (with ties broken deterministically) and the optimal expected reward is $\mu^*_n(\x_n) = \mu_n(\x_n|a^*_n)$. Let $\Delta = \max_{\x_n\in \mathcal{X}_n} \{\mu^*_n(\x_n) - \mu_n(\x_n|a_n\neq a^*_n)\}$ be the maximum reward difference between the optimal action and the non-optimal action over all context $\x_n \in \mathcal{X}_n$. Finally, we let $r_n(\x_n^k|\sigma^k_n)$ be the realized age-$n$ reward for video $k$ by using the learning algorithm $\sigma$. The expected regret by adopting a learning algorithm $\sigma_n$ is defined as
\begin{equation}\label{regret}
R_n(K) = \mathbb{E}\{\sum\limits_{k=1}^K  \mu^*_n(\x^k_n) - \sum\limits_{k=1}^K r_n(\x^k_n|\sigma^k_n)\}
\end{equation}

Our online learning algorithm will estimate the prediction rewards by selecting different actions and then choose the actions with best estimates based on past experience. One way to do this is to record the reward estimates without using the context/situational information. However, this could be very inefficient since for different contexts, the optimal actions can be very different. Another way is to maintain the reward estimates for each individual context $\x_n$ and select the action only based on these estimates. However, since the context space $\mathcal{X}_n$ can be very large, for a finite number $K$ of video instances, the number of videos with the same context $\x_n$ is very small. Hence it is difficult to select the best action with high confidence. Our learning algorithm will exploit the similarity information of contexts, partition the context space into smaller subspaces and learn the optimal action within each subspace. The key challenge is how and when to partition the subspace in an efficient way. Next, we propose an algorithm that adaptively partitions the context space according the arrival process of contexts.

\subsection{Online Popularity Prediction with Adaptive Partition}
In this subsection, we propose the online prediction algorithm with adaptive partition (Adaptive-Partition) that adaptively partitions the context space according to the context arrivals. This will be the key module of the Social-Forecast algorithm. For analysis simplicity, we normalize the context space to be $\mathcal{X}_n = [0,1]^d$. We call a $d$-dimensional hypercube which has sides of length $2^{-l}$ a level $l$ hypercube. Denote the partition of $\mathcal{X}_n$ generated by level $l$ hypercubes by $\mathcal{P}_l$. We have $|\mathcal{P}_l|=2^{ld}$. Let $\mathcal{P}:=\cup_{l=0}^\infty \mathcal{P}_l$ denote the set of all possible hypercubes. Note that $\mathcal{P}_0$ contains only a single hypercube which is $\mathcal{X}_n$ itself. For each instance arrival, the algorithm keeps a set of hypercubes that cover the context space which are mutually exclusive. We call these hypercubes {\it active} hypercubes, and denote the set of active hypercubes at instance $k$ by $\mathcal{A}_k$. Clearly, we have $\cup_{C\in \mathcal{A}_k} = \mathcal{X}_n$. Denote the active hypercube that contains $\x^k_n$ by $C_k$. Let $M_{C_k}(k)$ be the number of times context arrives to hypercube $C_{k}$ by instance $k$. Once activated, a level $l$ hypercube $C$ will stay active until the first instance $k$ such that $M_{C_k}(k) \geq  A2^{pl}$ where $p>0$ and $A>0$ are algorithm design parameters. When a hypercube $C_{k}$ of level $l$ becomes inactive, the hypercubes of level $l+1$ that constitute $C_{k}$, denoted by $\mathcal{P}_{l+1}(C_k)$, are then activated.

When a context $\x^k_n$ arrives, we first check to which active hypercube $C_k \in \mathcal{A}_k$ it belongs. Then we choose the action with the highest reward estimate $a_n = \arg\max\limits_a \bar{r}_{a,C_k}(k)$, where $\bar{r}_{a,C_k}(k)$ is the sample mean of the rewards collected from action $a$ in $C_k$ which is an activated hypercube at instance $k$. When the prediction reward is realized for instance $k$ (i.e. at the end of age $N$), we perform a {\it virtual update} for the reward estimates for all actions (see Figure \ref{virtualupdate}). The reason why we can perform such a virtual update for actions which are not selected is because the context transition over time is independent of our prediction actions and hence, the reward by choosing any action can still be computed even though it is not realized.

\begin{figure}
  \centering
  \includegraphics[scale=0.9]{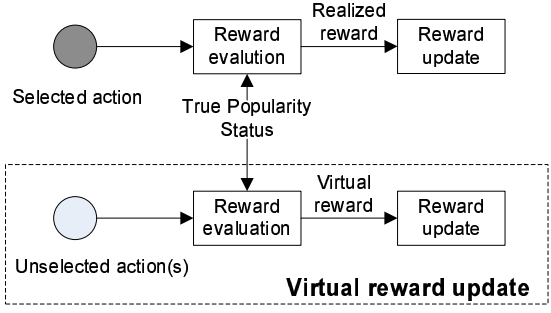}\\
  \caption{Illustration for virtual reward update in Adaptive Partition.}\label{virtualupdate}
\end{figure}

Algorithm \ref{PopPred-AP} provides a formal description for the Adaptive-Partition algorithm. Figure \ref{partition} illustrates the adaptive partition process of Adaptive-Partition algorithm. Next, we bound the regret by running the Adaptive-Partition algorithm.

\begin{algorithm}
\caption{Adaptive-Partition Algorithm}\label{PopPred-AP}
\begin{algorithmic}
\State Initialize $\mathcal{A}_1 = P_0$, $M_C(0) = 0$, $\bar{r}_{a,C}(0) = 0, \forall a, \forall C \in \mathcal{P}$.
\For {each video instance $k$}
\State Determine $C\in \mathcal{A}_k$ such that $\x^k_n \in C$.
\State Select $a_n = \arg\max\limits_a \bar{r}_{a,C}(k)$.
\State After the prediction reward is realized, update $\bar{r}_{a,C}(k+1)$ for all $a$.
\State Set $M_C(k)\leftarrow M_C(k-1) + 1$.
\If{ $M_C(k) \geq A2^{pl}$ }
\State Set $\mathcal{A}_{k+1} = (\mathcal{A}_k\backslash C)\cup \mathcal{P}_{l+1}(C)$ \EndIf
\EndFor
\end{algorithmic}
\end{algorithm}

\begin{figure*}[t]
  \centering
  \includegraphics[scale=1.1]{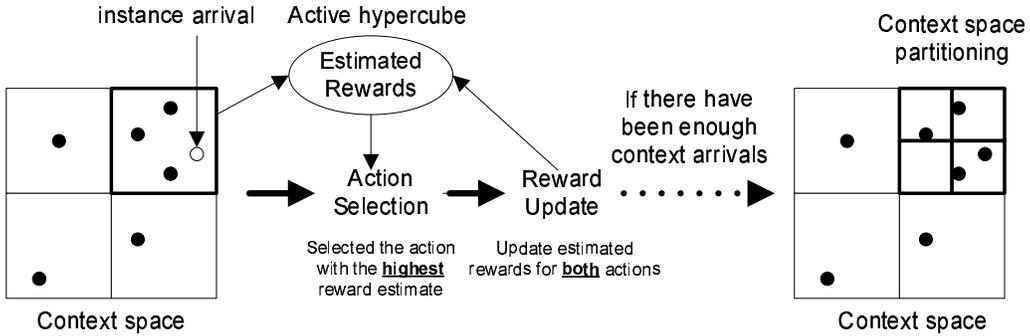}\\
  \caption{The context space partitioning of the Adaptive-Partition algorithm.}\label{partition}
\end{figure*}

In order to get the regret bound of the Adaptive-Partition algorithm, we need to consider how many hypercubes of each level is formed by the algorithm up to instance $K$. The number of such hypercubes explicitly depends on the context arrival process. Therefore, we investigate the regret for different context arrival scenarios.
\begin{definition*}
We call the context arrival process the {\bf the worst-case arrival process} if it is uniformly distributed inside the context space, with minimum distance between any two context samples being $K^{-1/d}$, and {\bf the best-case arrival process} if $\x^k \in C, \forall k$ for some level $\lceil (\log_2(K)/p\rceil +1$ hypercube $C$.
\end{definition*}

In Theorem 2, we determine the finite time, uniform regret bound for the Adaptive-Partition algorithm. The complete regret analysis and proofs can be found in the appendix.

\begin{theorem}\label{thmreg}
\begin{itemize}
  \item For the worst case arrival process, if $p = \frac{3\alpha + \sqrt{9\alpha^2+8\alpha d}}{2}$, then $R_n(K)=O(K^{\frac{d+\alpha/2 + \sqrt{9\alpha^2 + 8\alpha d}/2}{d+3\alpha/2 + \sqrt{9\alpha^2 + 8\alpha d}/2}})$.
  \item For the best case arrival process, if $p = 3\alpha$, then $R_n(K) = O(K^{2/3})$.
\end{itemize}
\end{theorem}
\begin{proof}
See Appendix.
\end{proof}
The regret bounds proved in Theorem 2 are sublinear in $K$ which guarantee convergence in terms of the average reward, i.e. $\lim_{K\to \infty} \mathbb{E}[R_n(K)]/K = 0$. Thus our online prediction algorithm makes the optimal predictions as sufficiently many videos instances have been seen. More importantly, the regret bound tells how much reward would be lost by running our learning algorithm for any finite number $K$ of videos arrivals. Hence, it provides a rigorous characterization on the learning speed of the algorithm.

\subsection{Learning the Complete Policy $\pi$}
In the previous subsection, we proposed the Adaptive-Partition algorithm to learn the optimal policy $\pi^*_n(\pi_{-n})$ by fixing $\pi_{-n}$. We now present in Algorithm \ref{complete} the Social-Forecast algorithm that learns the complete policy.

\begin{algorithm}
\caption{Social-Forecast Algorithm}\label{complete}
\begin{algorithmic}
\For {each video instance $k$}
\For {each age $n = 1$ to $N$}
\State    Get context information $\x^k_n$.
\State    Select $a^k_n$ according to Adaptive-Partition.
\State    Perform context partition using Adaptive-Partition.
\EndFor
\State    Popularity status $s^k$ is realized.
\For {each age $n = 1$ to $N$}
\State    Compute the age-dependent reward $r^k_n$.
\State    Update reward estimates using Adaptive-Partition.
\EndFor
\EndFor
\end{algorithmic}
\end{algorithm}

Social-Forecast learns all age-dependent policies $\pi_n, \forall n$ simultaneously. For a given age $n$, since $\pi_{-n}$ is not fixed to be the optimal policy $\pi^{opt}_{-n}$ during the learning process, the learned policy $\pi_{n}$ may not be the global optimal $\pi^{opt}_n$. However, as we have shown in Section IV, in order to determine $\pi^{opt}_n$, only the policies for ages greater than $n$, i.e.  $\pi^{opt}_{n^+}$ need to be determined. Thus even though we are learning $\pi_n, \forall n$ simultaneously, the learning problem of $\pi_N$ is not affected and hence, $\pi^{opt}_N$ will be learned with high probability after a sufficient number of video arrivals. Once $\pi^{opt}_N$ is learned with high probability, $\pi^{opt}_{N-1}$ can also be learned with high probability after an additional number of video arrivals. By this induction, such a simultaneous learning algorithm can still learn the global optimal complete policy with high probability. In the experiments we will show the performance of this algorithm in practice.

\subsection{Complexity of Social-Forecast}
For each age of one video instance arrival, Social-Forecast needs to do one comparison operation and one update operation on the estimated reward of each forecast action. It also needs to update the counting of context arrivals to the current context subspace and perform context space partitioning if necessary. In sum, the time complexity has the order $O(|\mathcal{S}|N)$ for each video instance and $O(|\mathcal{S}|NK)$ for $K$ video arrivals. Since the maximum age $N$ of interest and the popularity status space is given, the time complexity is linear in the number of video arrivals $K$. The Social-Forecast algorithm maintains for each \textit{active} context subspace reward estimates of all forecast actions. Each partitioning creates $2^d - 1$ more \textit{active} context subspaces and the number of partitioning is at most $K/A$. Thus the space complexity for $K$ video arrivals is at most $O(2^d N K/A)$. Since the context space dimension $d$ and the algorithm parameter $A$ are given and fixed, the space complexity is at most linear in the number of video arrivals $K$.

\section{Experiments}
In this section we evaluate the performance of the proposed Social-Forecast algorithm. We will first examine the unique propagation characteristics of videos shared through social media. Then we will use these as the context (and situational) information for our proposed online prediction algorithm. Our experiments are based on the dataset that tracks the propagation process of videos shared on RenRen (\texttt{www.renren.com}), which is one of the largest Facebook-like online social networks in China. We set one period to be 2 hours and are interested in predicting the video popularity by 100 periods (8.3 days) after its initiation. In most of our experiments, we will consider a binary popularity status space \{Popular, Unpopular\} where ``Popular'' is defined for videos whose total number of views exceeds 10000. However, we also conduct experiments on a more refined popularity status space in Section VI(F).

The prediction reward function that we use is $U(a^k_n, s^k, n) = \theta(a^k_n, s^k) + \lambda \psi(n)$. For the case of binary popularity status space, the accuracy reward function $\theta$ is chosen as follows
\begin{equation}
\theta(a^k_n, s^k) =\left\{
\begin{array}{ll}
1, &\textrm{if}~a^k_n = s^k = \textrm{Unpopular}\\
w, &\textrm{if}~a^k_n = s^k = \textrm{Popular}\\
0, &\textrm{if}~a^k_n \neq s^k
\end{array}\right.
\end{equation}
where $w > 0$ is fixed reward for correctly predicting popular videos and hence controls the relative importance of true positive and true negative. The timeliness reward function $\psi$ is simply taken as $\psi(n) = N - n$. Recall that the prediction reward function is a combination of the two and we use $\lambda > 0$ to trade-off accuracy and timeliness. In the experiments, we will vary both $w$ and $\lambda$ to investigate their impacts on the prediction performance. Note that we use these specific reward functions in this experiment but other reward functions can easily be adopted in our algorithm.

\subsection{Video propagation characteristics}
A RenRen user can post a link to a video taken by him/herself or from an external video sharing website such as Youtube. The user, referred to as an {\it initiators}~\cite{Li}, then starts the sharing process. The friends of these initiators can find this video in their ``News Feed''. Some of them may watch this video and some may re-share the video to their own friends. We call the users who watched the shared video {\it viewers} and those who re-shared the video {\it spreaders}. Since spreaders generally watched the video before re-shared it, most of them are also viewers. In the experiment, we will use two characteristics of videos promoted by social media as the context (and situational) information for our algorithm. The first is the initiator's {\it Branching Factor (BrF)}, which is the number of viewers who directly follow the initiator. The second is the {\it Share Rate (ShR)}, which is the ratio of the viewers that re-share the video after watching it. Figure \ref{context} shows the evolution of the number of views, the BrF and the ShR for three representative videos over 100 periods. Among these three videos, video 1 is an unpopular video while video 2 and video 3 are popular videos, which become popular at age 37 and age 51, respectively. We analyze the differences between popular and unpopular videos as follows.
\begin{itemize}
  \item {\it Video 1 vs Video 2}. The ShRs of both videos are similar. The BrF of video 2 is much larger than that of video 1. This indicates that video 1 may be initiated by users with a large number of friends, e.g. celebrities and pubic accounts. Thus, videos with larger BrF potentially will achieve popularity in the future.
  \item {\it Video 1 vs Video 3}. The BrFs of both videos are low (at least before video 3 becomes popular). Video 3 has a much larger ShR than video 1. This indicates that video 3 is being shared with high probability and thus, videos with larger ShR will potentially become popular in the future.
\end{itemize}

The above analysis shows that BrF and ShR are good situational metrics for videos promoted by social media. Therefore we will use these two metrics in addition to the total and per-period numbers of views as the context information for our proposed online prediction algorithms. Nevertheless, our algorithms are general enough to take other situational metrics to further improve the prediction performance, e.g. the type of the videos, the number of spreaders, the propagation topology etc.

\begin{figure}
  \centering
  \includegraphics[scale=0.6]{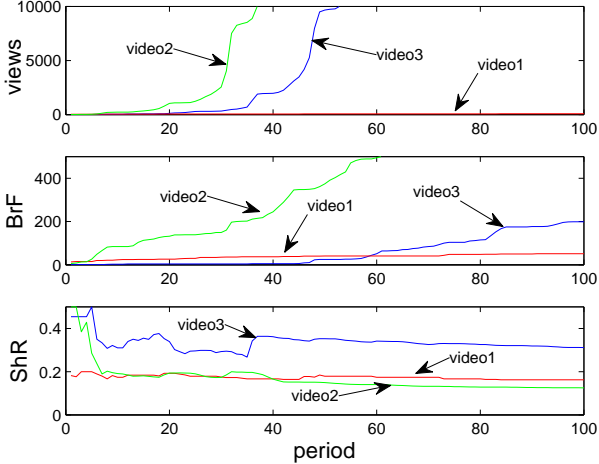}\\
  \caption{Popularity evolution of 3 representative videos.}\label{context}
\end{figure}

\subsection{Benchmarks}
We will compare the performance of our online prediction algorithm with four benchmarks.
\begin{itemize}
  \item {\bf All Unpopular (AU)}. The first benchmark is a naive scheme which predicts that all videos are not popular at age 1. This is equivalent to $a^k_1 = \textrm{Unpopular},\forall k$.
  \item {\bf All Popular (AP)}. The second benchmark is another naive scheme which makes the prediction at age 1 that the video will become popular in the future. This is equivalent to take the action $a^k_1 = \textrm{Popular},\forall k$.
  \item {\bf View-based Prediction (VP)}. The third benchmark is a conventional view-based prediction algorithm based on~\cite{Szabo}. It uses training sets to establish log-linear correlations between the early number of views and the later number of views. Since this algorithm does not explicitly consider timeliness in prediction, we will investigate different versions that make predictions at different ages. Intuitively, the time when the prediction is made has oppositive affects on the prediction accuracy and timeliness. A later prediction predicts the video with higher confidence but is less timely.
  \item {\bf Perfect Prediction}. The last benchmark provides the best prediction results: for each unpopular video, it predicts unpopular at age 1; for each popular video, it predicts popular at age 1. Since this benchmark generates the highest possible prediction reward, we normalize the rewards achieved by other algorithms with respect to this reward.
\end{itemize}

\subsection{Performance comparison}
In this subsection, we compare the prediction performance of our proposed algorithm with the benchmarks. This set of experiments are carried out on a set of 10000 video links, among which 10\% are popular videos. The videos were initiated in sequence and thus, initially we do not have any knowledge of the videos or video popularity evolution patterns. For the VP algorithm, we use three versions, labeled as VP-25, VP-50, VP-75, in which the prediction is made at age 25, 50, 75, respectively.

Table \ref{comparisonW} records the normalized prediction rewards obtained by our proposed algorithm and the benchmarks for $\lambda = 0.010$ and $w = 5, 10, 15$. The trade-off parameter $\lambda$ for accuracy and timeliness is set to be small because the lifetime $N$ is large. The Social-Forecast algorithm is labeled by SF.
\begin{itemize}
  \item For AU and AP, even though their accuracy is expected to be bad, they will obtain full timeliness rewards because they make the predictions at the first age for each video. However, since their prediction accuracy is low, their overall prediction rewards are the lowest among all algorithms. The reward achieved by AU is decreasing in $w$ and the reward achieved by AP is increasing in $w$. This is because a larger $w$ assigns higher importance to correct prediction for popular videos and the fact that AP predicts all popular videos correctly and AU predicts all unpopular videos correctly.
  \item VP algorithms achieve better prediction rewards than AU and AP. As can be seen from the table, an early prediction generates higher rewards because a large portion of the reward is derived from the timeliness of the prediction. The reward achieved by VP-25 is decreasing in $w$ while those achieved by VP-50 and VP-75 are increasing in $w$. This implies that VP-25 has a better performance on predicting unpopular videos than predicting popular videos.
  \item The proposed algorithm Social-Forecast generates significantly higher prediction rewards than all benchmark algorithms. Its performance is not sensitive to the specific value of $w$ which implies that it is able to predict both popular and unpopular videos very accurately and in a timely manner.
\end{itemize}

\begin{table}
\centering
\caption{Comparison of normalized prediction reward with varying $w$}\label{comparisonW}
\begin{tabular}{|c|c|c|c|c|c|c|}
  \hline
  ~ & AU & AP & VP-25 & VP-50 & VP-75 & SF \\
  \hline
  $w = 5$ & 0.795 & 0.622 & 0.831 & 0.763 & 0.643 & 0.986 \\
  \hline
  $w = 10$ & 0.663 & 0.685 & 0.823 & 0.803 & 0.671 & 0.983 \\
  \hline
  $w = 15$ & 0.549 & 0.740 & 0.814 & 0.837 & 0.691 & 0.981 \\
  \hline
\end{tabular}
\end{table}

Next, we fix $w$ and vary $\lambda$. Table \ref{comparisond} records the normalized prediction rewards obtained by our proposed algorithm and the benchmarks for $W = 10$ and $\lambda = 0.005,0.010,0.015$. Several points are worth discussing:
\begin{itemize}
  \item The rewards obtained by both and AU and AP are increasing in $\lambda$. This is because both benchmarks derive full reward from the timeliness prediction since they make prediction at the first age for all videos.
  \item The rewards obtained by all three versions of VP are decreasing in $\lambda$. This suggests the rewards are mainly derived from prediction accuracy but the VP algorithms are not able to make the prediction in a timely manner.
  \item Our proposed Social-Forecast algorithm significantly outperforms all other benchmark algorithms and achieve close-to-optimal rewards for all values of $\lambda$.
\end{itemize}

\begin{table}
\centering
\caption{Comparison of normalized prediction reward with varying $\lambda$}\label{comparisond}
\begin{tabular}{|c|c|c|c|c|c|c|}
  \hline
  ~ & AU & AP & VP-25 & VP-50 & VP-75 & SF \\
  \hline
  $\lambda = 0.005$ & 0.601 & 0.612 & 0.835 & 0.862 & 0.757 & 0.980 \\
  \hline
  $\lambda = 0.010$ & 0.663 & 0.685 & 0.823 & 0.803 & 0.671 & 0.983 \\
  \hline
  $\lambda = 0.015$ & 0.701 & 0.737 & 0.816 & 0.762 & 0.613 & 0.983 \\
  \hline
\end{tabular}
\end{table}

We also investigate the achieved predication accuracy in terms of true positive rate and true negative rate. We define the true positive rate as the ratio of correctly predicted videos among all popular videos and the true negative rate as the ratio of correctly predicted videos among all unpopular videos. Table \ref{TPTN} illustrates the true positive rates and true negative rates achieved by different algorithms. As can be seen from the table, in general prediction at a later age for the VP algorithms improves the accuracy. However, it is not always the case since the true negative rate achieved by VP-75 is low. This suggests that the correlation used by VP-75 for unpopular videos does not accurately reflect the true popularity evolution trend. Instead, our proposed Social-Forecast is able to achieve both a high true positive rate and a high true negative rate, at the same time predicting in a timely manner.

\begin{table}
\centering
\caption{True Positive and True Negative.}\label{comparison}
\begin{tabular}{|c|c|c|c|c|c|c|}
  \hline
  ~ & AU & AP & VP-25 & VP-50 & VP-75 & SF \\
  \hline
  True Positive & 0 & 1 & 0.918 & 0.917 & 0.995 & 0.983 \\
  \hline
  True Negative & 1 & 0 & 0.804 & 0.994 & 0.789 & 0.972 \\
  \hline
\end{tabular}\label{TPTN}
\end{table}

\subsection{Learning performance}
Our proposed Social-Forecast algorithm is an online algorithm and does not require any prior knowledge of the video popularity evolution patterns. Hence, it is important to investigate the prediction performance during the learning process. Our analytic results have already provided sublinear bounds on the prediction performance for any given number of video instances which guarantee the convergence to the optimal prediction policy. Now, we show how much prediction reward that we can achieve during the learning process in experiments. Figure \ref{learning} shows the normalized prediction reward of Social-Forecast as the number of videos instances increases.  As more video instances arrive, our algorithm learns better the optimal prediction policy and hence, the prediction reward improves with the number of video instances. In particular, the proposed prediction algorithm is able to achieve more than 90\% of the best possible reward even with a relatively small number of video instances.

\begin{figure}
  \centering
  \includegraphics[scale=0.6]{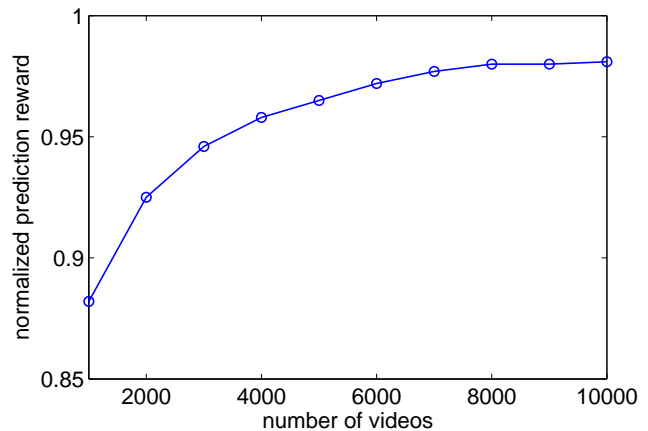}\\
  \caption{Prediction performance during the learning process.}\label{learning}
\end{figure}

\subsection{Choices of lifetime $N$}
So far in our analysis, we treated the prediction reference age $N$ as fixed. In practice, this is a parameter that can be set by the content providers, the advertisers and the web hosts depending on their specific requirements. In this experiment, we show the performance of our algorithm for different choices of $N$. Table \ref{comparisonN} provides the prediction rewards obtained by different algorithms for $N = 60, 70, 80, 90, 100$ when $d = 0.010$ and $w = 10$. In all experiments, the proposed algorithm achieves significant higher prediction rewards than the benchmarks. This shows that our methodology and associated algorithm is general and applicable for problems with different requirements.

\begin{table}
\centering
\caption{Impact of reference age $N$}\label{comparisonN}
\begin{tabular}{|c|c|c|c|c|}
  \hline
  ~ & AU & AP & VP-25 &  SF \\
  \hline
  $N = 60$ & 0.600 & 0.569 & 0.722 & 0.920  \\
  \hline
  $N = 70$ & 0.612 & 0.655 & 0.801 & 0.992  \\
  \hline
  $N = 80$ & 0.645 & 0.657 & 0.814 & 0.989  \\
  \hline
  $N = 90$ & 0.642 & 0.678 & 0.816 & 0.986  \\
  \hline
  $N = 100$ & 0.663 & 0.685 & 0.823 & 0.983  \\
  \hline
\end{tabular}
\end{table}

\subsection{More refined popularity prediction}
In the previous experiments, we considered a binary popularity status space. Nevertheless, our proposed popularity prediction methodology and associated algorithm can also be applied to predict popularity in a more refined space. In this experiment, we consider a refined popularity status space \{High Popularity, Medium Popularity, Low Popularity\} where ``High Popularity'' is defined for videos with more than 10000 views, ``Medium Popularity'' for videos with views between 2000 and 10000, and ``Low Popularity'' for videos with views below 2000. The portions of High, Medium and Low videos are 10\%, 30\% and 60\%, respectively. We set rewards for correctly predicting High, Medium, Low videos to be 10, 5 and 1, respectively. The proposed Social-Forecast algorithm is compared against the view-based algorithm VP performed at different prediction ages. Table \ref{refinement} illustrates the normalized rewards obtained by different algorithms for $\lambda = 0.005, 0.010, 0.015$. As can been seen from the table, the rewards obtained by all algorithms decrease compared with the binary popularity status case since prediction becomes more difficult. However, the performance improvement of Social-Forecast against VP becomes even larger. This suggests that our algorithm, which explicitly considers the situational/contextual information, is able to achieve a higher performance gain against view-based approaches for more refined popularity prediction.

\begin{table}
\centering
\caption{Performance comparison for refined popularity space.}\label{refinement}
\begin{tabular}{|c|c|c|c|c|}
  \hline
  ~ & VP-25 & VP-50 & VP-75 & SF \\
  \hline
  $\lambda = 0.005$ & 0.487 & 0.601 & 0.627 & 0.938 \\
  \hline
  $\lambda = 0.010$ & 0.493 & 0.580 & 0.578 & 0.928 \\
  \hline
  $\lambda = 0.015$ & 0.492 & 0.558 & 0.545 & 0.920 \\
  \hline
\end{tabular}
\end{table}

\section{Conclusions and Future Work}
In this paper, we have proposed a novel, systematic and highly-efficient
online popularity forecasting algorithm for videos promoted by social media. We have shown that by incorporating situational and contextual information, significantly better prediction performance can be achieved than existing approaches which disregard this information and only consider the number of times that videos have been viewed so far. The proposed Social-Forecast algorithm does not require prior knowledge of popularity evolution or a training set and hence can operate easily and successfully in online, dynamically-changing environments such as social media. We have systematically proven sublinear regret bounds on the performance loss incurred by our algorithm due to online learning. Thus Social-Forecast guarantees both short-term performance as well as its asymptotic convergence to the optimal performance in the long term.

This paper considered a single learner who observes the propagation patterns of videos promoted by one social media. One important future work direction is to extend to scenarios where there are multiple distributed learners (e.g. multiple advertisers, content producers and web hosts) who have access to multiple different social medias or different sections of one social media. In such scenarios, significant improvement is expected by enabling cooperative learning among the distributed learners~\cite{Tekin13}. The challenges in these scenarios are how to design efficient cooperative learning algorithms with low communication complexity~\cite{Xu13} and, when the distributed learners are self-interested and have conflicting goals, how to incentivize them to participate in the cooperative learning process using, e.g. rating mechanisms~\cite{Xu14}~\cite{XuNetecon}. Finally, while this paper has studied the specific
problem of online prediction of video popularity based on contextual
and situational information, our methodology and associated algorithm
can be easily adapted to predict other trends in social media (such
as identifying key influencers in social networks, the potential for
becoming viral of contents or tweets, identifying popular or relevant
content, providing recommendations for social TV etc.).

\appendix
In this appendix, we analyze the learning regret of the Adaptive-Partition algorithm. To facilitate the analysis, we artificially create two learning steps in the algorithms: for each instance $k$, it belongs to either a {\it virtual exploration} step or a {\it virtual exploitation} step.  Let $M_C(k)$ be the number of context arrivals in $C$ by video instance $k$. Given a context $\x^k_n \in C$, which step the instance $k$ belongs to depends on $M_C(k)$ and a deterministic function $D(k)$. If $M_C(k) \leq D(k)$, then it is in a virtual exploration step; otherwise, it is in a virtual exploitation step. Notice that these steps are only used in the analysis; in the implementation of the algorithm, these different steps do not exist and are not needed.

We introduce some notations here. Let $\mathcal{E}_{a,C}(k)$ be the set of rewards collected from action $a$ by instance $k$ for hypercube $C$. For each hypercube $C$ let $a^*(C)$ be the action which is optimal for the center context of that hypercube, and let $\bar{\mu}_{a,C}:=\sup_{\x\in C} \mu(\x|a)$ and $\underline{\mu}_{a,C} :=\inf_{\x\in C} \mu(\x|a)$. For a level $l$ hypercube $C$, the set of suboptimal action is given by
\begin{equation}
\mathcal{L}_{C,l,B}:=\{a:\underline{\mu}_{a^*,C} - \bar{\mu}_{a,C} > BL d^{\alpha/2} 2^{-l\alpha}\}
\end{equation}
where $B>0$ is a constant.

The regret can be written as a sum of three components:
\begin{equation}
R(K) = \mathbb{E}[R_e(K)] + \mathbb{E}[R_s(K)] + \mathbb{E}[R_n(K)]
\end{equation}
where $R_e(K)$ is the regret due to virtual exploration steps by instance $K$, $R_s(K)$ is the regret due to sub-optimal action selection in virtual exploitation steps by instance $K$ and $R_n(K)$ is the regret due to near-optimal action selections in virtual exploitation steps by instance $K$. The following series of lemmas bound each of these terms separately.

We start with a simple lemma which gives an upper bound on the highest level hypercube that is active at any instance $k$.

\begin{lemma}
All the active hypercubes $\mathcal{A}_k$ at instance $k$ have at most a level of $(\log_2k)/p+1$.
\end{lemma}
\begin{proof}
Let $l+1$ be the level of the highest level active hypercube. We must have $\sum\limits_{j=1}^l A 2^{pj} < k$, otherwise the highest level active hypercube will be less than $l+1$. We have for $k/A > 1$,
\begin{align}
A\frac{2^{p(l+1)-1}}{2^p-1} < k \Rightarrow 2^{pi} < \frac{k}{A} \Rightarrow i < \frac{\log_2(k)}{p}
\end{align}
\end{proof}

The next three lemmas bound the regrets for any level $l$ hypercube. .

\begin{lemma}
If $D(k)= k^z \log k$. Then, for any level $l$ hypercube the regret due to virtual explorations by instance $k$ is bounded above by $\Delta(k^z \log k + 1)$.
\end{lemma}
\begin{proof}
Since the instance $k$ belongs to a virtual exploration step if and only if $M_C(k) \leq D(k)$, up to instance $K$, there can be at most $\lceil k^z\log k\rceil$ virtual exploration steps for one hypercube. Therefore, the regret is bounded by $\Delta(k^z \log k + 1)$.
\end{proof}

\begin{lemma}
Let $B = \frac{2}{Ld^{\alpha/2}2^{-\alpha}} + 2$. If $p > 0, 2\alpha/p \leq z < 1$, $D(k) = k^z \log k$, then for any level $l$ hypercube $C$, the regret due to choosing suboptimal actions in virtual exploitation steps, i.e. $\mathbb{E}[R_{C,s}(K)]$, is bounded above by $2\beta_2$.
\end{lemma}
\begin{proof}
Let $\Omega$ denote the space of all possible outcomes, and $w$ be a sample path. The event that the algorithm virtually exploits in $C$ at instance $k$ is given by
\begin{align*}
\mathcal{W}_{C}(k):=\{w: M_C(k) > D(k), \x^k_n \in C, C\in\mathcal{A}_k\}
\end{align*}
We will bound the probability that the algorithm chooses a suboptimal arm in an virtual exploitation step in $C$, and then bound the expected number of times a suboptimal action is chosen by the algorithm. Recall that loss in every step is at most $1$. Let $\mathcal{V}_{a,C}(k)$ be the event that a suboptimal action is chosen. Then
\begin{align*}
\mathbb{E}[R_{C,s}(K)]\leq \sum\limits_{k=1}^K\sum\limits_{a \in \mathcal{L}_{C,l,B}}P(\mathcal{V}_{a,C}(k), \mathcal{W}_C(k))
\end{align*}

For any $a$, we have
\begin{align*}
&\{\mathcal{V}_{a,C}(k), \mathcal{W}_C(k)\}\\
\subset &\{\bar{r}_{a,C}(k) \geq \bar{\mu}_{a,C} + H_k, \mathcal{W}_C(k)\}\\
&\cup \{\bar{r}_{a^*,C}(k) \leq \underline{\mu}_{a^*,C} - H_k, \mathcal{W}_C(k)\}\\
&\cup \{\bar{r}_{a,C}(k) \geq \bar{r}_{a^*,C}(k),\bar{r}_{a,C}(k) < \bar{\mu}_{a,C} + H_k,\\
&\bar{r}_{a^*,C}(k) > \underline{\mu}_{a^*,C} - H_k, \mathcal{W}_C(k)\}
\end{align*}
for some $H_k > 0$.  This implies
\begin{align*}
&P(\mathcal{V}_{a,C}(k), \mathcal{W}_C(k))\\
\leq & P(\bar{r}^{best}_{a,C}(M_{C}(k)) \geq \bar{\mu}_{a,C} + H_k + Ld^{\alpha/2}2^{-l\alpha}, \mathcal{W}_C(k))\\
+&P(\bar{r}^{worst}_{a^*,C}(M_{C}(k)) \leq \underline{\mu}_{a^*,C} - H_k - Ld^{\alpha/2}2^{-l\alpha}, \mathcal{W}_C(k))\\
+& P(\bar{r}^{best}_{a,C}(M_{C}(k)) \geq \bar{r}^{worst}_{a^*,C}(M_{C}(k)), \\ &~~~~~~\bar{r}^{best}_{a,C}(M_{C}(k)) < \bar{\mu}_{a,C} + H_k, \\
&~~~~~~\bar{r}^{worst}_{a^*,C}(M_{C}(k)) > \underline{\mu}_{a^*,C} - H_k, \mathcal{W}_C(k))
\end{align*}
Consider the last term in the above equation. In order to make the right-hand side to be 0, we need, $2H_k \leq (B-2)Ld^{\alpha/2}2^{-l\alpha}$. This holds when $2H_k \leq (B-2)Ld^{\alpha/2}2^{-\alpha}k^{-\alpha/p}$.
For $H_k = k^{-z/2}$, $z \geq 2\alpha/p$ and $B = \frac{2}{Ld^{\alpha/2}2^{-\alpha}}+2$, the last term is $0$. By using a Chernoff-Hoeffding bound, for any $a \in \mathcal{L}_{C,l,B}$, since on the event $\mathcal{W}_C(k)$, $M_{C}(k) \geq k^z\log k$, we have
\begin{align*}
&P(\bar{r}^{best}_{a,C}(M_{C}(k)) \geq \bar{\mu}_{a,C} + H_k, \mathcal{W}_C(k))\\
\leq & e^{-2(H_k)^2 k^z\log k} \leq \frac{1}{k^2}
\end{align*}
and
\begin{align*}
&P(\bar{r}^{worst}_{a^*,C}(M_{C}(k)) \leq \underline{\mu}_{a^*,C} - H_k, \mathcal{W}_C(k))\\
\leq &e^{-2(H_k)^2 k^z\log k} \leq \frac{1}{k^2}
\end{align*}
Therefore, $\mathbb{E}[R_{C,s}(K)] \leq 2\beta_2$.
\end{proof}

\begin{lemma}
Let $B = \frac{2}{Ld^{\alpha/2}2^{-\alpha}}+2$. If $p > 0, 2\alpha/p \leq z < 1$, $D(k) = k^z \log k$, then for any level $l$ hypercube $C$, the regret due to choosing near optimal actions in virtual exploitation steps, i.e. $\mathbb{E}[R_{C,n}(K)]$, is bounded above by $2ABLd^{\alpha/2}2^{(p-\alpha)l}$.
\end{lemma}
\begin{proof}
The one-step regret of any near optimal action $a$ is bounded by $2BLd^{\alpha/2}2^{-l\alpha}$. Since $C$ remains active for at most $A2^{pl}$ context arrivals, we have
\begin{align}
\mathbb{E}[R_{C,n}(K)]\leq 2ABLd^{\alpha/2}2^{(p-\alpha)l}
\end{align}
\end{proof}

Now we are ready to prove Theorem 2.
\begin{proof}
We let $B = \frac{2}{Ld^{\alpha/2}2^{-\alpha}}+2$.

Consider the worst-case. It can be shown that in the worst case the highest level hypercube has level at most $1 + \log_{2^{p+d}}K$. The total number of hypercubes is bounded by
\begin{align}
\sum\limits_{l=0}^{1 + \log_{2^{p+d}}K}2^{dl} \leq 2^{2d} K^{\frac{d}{d+p}}
\end{align}

We can calculate the regret from choosing near optimal action as
\begin{align}
&\mathbb{E}[R_n(K)] \leq 2ABLd^{\alpha/2}\sum\limits_{l=0}^{1 + \log_{2^{p+d}}K}2^{(p-\alpha)l}\\
\leq & 2ABLd^{\alpha/2} 2^{2(d+p-\alpha)} K^{\frac{d+p-\alpha}{d+p}}
\end{align}

Since the number of hypercubes is $O(K^{\frac{d}{d+p}})$, regret due to virtual explorations is $O(K^{\frac{d}{d+p}+z}\log K)$, while regret due to suboptimal selection is $O(K^{\frac{d}{d+p}+z})$, for $z \geq \frac{2\alpha}{p}$. These three terms are balanced when $z = 2\alpha/p$ and $\frac{d+p-\alpha}{d+p} = \frac{d}{d+p} + z$. Solving for $p$ we get
\begin{align}
p = \frac{3\alpha + \sqrt{9\alpha^2+8\alpha d}}{2}
\end{align}
Substituting these parameters and summing up all the terms we get the regret bound.

Consider the best case, the number of activated hypercubes is upper bounded by $\log_2 K/p + 1$, and by the property of context arrivals all the activated hypercubes have different levels. We calculate the regret from choosing near optimal arm as
\begin{align}
&\mathbb{E}[R_n(K)] \leq 2ABLd^{\alpha/2}\sum\limits_{l=0}^{1 + \log_{2}K/p}2^{p-\alpha}l \\
\leq &2ABLd^{\alpha/2} \frac{2^{2(p-\alpha)}}{2^{p-\alpha}}K^{\frac{p-\alpha}{p}}
\end{align}
The terms are balanced by setting $z = 2\alpha/p$, $p = 3\alpha$.
\end{proof}


%

%
%
%
%
%

\ifCLASSOPTIONcaptionsoff
  \newpage
\fi

\end{document}